\documentclass{article}
\usepackage{kr}

\usepackage{times}
\usepackage{soul}
\usepackage{url}
\usepackage[hidelinks]{hyperref}
\usepackage[utf8]{inputenc}
\usepackage[small]{caption}
\usepackage{graphicx}
\usepackage{amsmath}
\usepackage{amsthm}
\usepackage{booktabs}
\usepackage{algorithm}
\usepackage{algorithmic}
\newtheorem{theorem}{Theorem}
\newtheorem{example}[theorem]{Example}
\newtheorem{definition}[theorem]{Definition}
\newtheorem{proposition}[theorem]{Proposition}
\newtheorem{lemma}[theorem]{Lemma}
\newtheorem{corollary}[theorem]{Corollary}
\newtheorem{remark}[theorem]{Remark}

\title{On the Complexity and Properties of Preferential Propositional Dependence Logic}

\author{%
    Kai Sauerwald$^1$\and Arne Meier$^2$\and Juha Kontinen$^3$
    \affiliations
    $^1$Artificial Intelligence Group, University of Hagen, Hagen, Germany\\
        $^2$Theoretical Computer Science, Leibniz University Hannover, Hannover, Germany\\
    $^3$Department of Mathematics and Statistics, University of Helsinki, Helsinki, Finland
}
\usepackage{dlcobook}
\usepackage{xspace}
\usepackage{mathrsfs} 
\usepackage{mathtools}
\usepackage{amssymb,amsmath}
\usepackage{comment}
\usepackage{stmaryrd}
\usepackage{csquotes}
\usepackage[UKenglish]{babel}
\usepackage{enumerate}
\usepackage{microtype}

\usepackage{pifont}%
\newcommand{\xmark}{\ding{55}}%

\newcommand{\minOf}[2]{\ensuremath{\min(#1,#2)}}
\renewcommand{\tuple}[1]{\ensuremath{\langle{#1}\rangle}}
\newcommand{\statesOf}[1]{\ensuremath{S(#1)}}
\newcommand{\modelsOf}[1]{\ensuremath{\llbracket #1\rrbracket}}

\newcommand{\Int}{\ensuremath{X}}   %

\newcommand{\pmW}{\ensuremath{\mathbb{W}}}
\newcommand{\pmS}{\ensuremath{\mathcal{S}}}
\newcommand{\pmL}{\ensuremath{\ell}}
\newcommand{\pmP}{\ensuremath{\prec}}

\newcommand{\pmWsub}{\ensuremath{\pmW_{\!\text{sub}}}}
\newcommand{\pmSsub}{\ensuremath{\pmS_{\text{sub}}}}
\newcommand{\pmLsub}{\ensuremath{\pmL_{\text{sub}}}}
\newcommand{\pmPsub}{\ensuremath{\pmP_{\text{sub}}}}

\newcommand{\pmWsup}{\ensuremath{\pmW_{\!\text{sup}}}}
\newcommand{\pmSsup}{\ensuremath{\pmS_{\text{sup}}}}
\newcommand{\pmLsup}{\ensuremath{\pmL_{\text{sup}}}}
\newcommand{\pmPsup}{\ensuremath{\pmP_{\text{sup}}}}

\newcommand{\pmWpq}{\ensuremath{\pmW_{\!\text{pq}}}}
\newcommand{\pmSpq}{\ensuremath{\pmS_{\text{pq}}}}
\newcommand{\pmLpq}{\ensuremath{\pmL_{\text{pq}}}}
\newcommand{\pmPpq}{\ensuremath{\pmP_{\text{pq}}}}

\DeclareMathOperator{\nmableitW}{\nmableit_{\mkern-2mu\pmW}}
\newcommand{\nmableitWparam}[1]{\nmableit_{\mkern-2mu#1}}
\DeclareMathOperator{\notnmableitW}{\notnmableit_{\mkern-2mu\pmW}}
\newcommand{\notnmableitWparam}[1]{\notnmableit_{\mkern-2mu#1}}

\ifx\nmableit\undefined
\makeatletter
\ifx\centernot\undefined
\newcommand*{\centernot}{%
	\mathpalette\@centernot
}
\fi
\def\@centernot#1#2{%
	\mathrel{%
		\rlap{%
			\settowidth\dimen@{$\m@th#1{#2}$}%
			\kern.5\dimen@
			\settowidth\dimen@{$\m@th#1=$}%
			\kern-.5\dimen@
			$\m@th#1\not$%
		}%
		{#2}%
	}%
}
\makeatother
\DeclareRobustCommand\nmableitSymb{\mathrel{|\mkern-8.5mu\sim}} %
\newcommand{\nmableit}{\nmableitSymb} %
\newcommand{\notnmableit}{\centernot\nmableitSymb} %
\fi

\newcommand{\problemdef}[3]{%
\begin{center}
\begin{tabular}{l@{\hskip 0.2cm}p{6.1cm}}\toprule
\textsf{\bfseries Problem:}& #1 \\\midrule
\textsf{\bfseries Input:}& #2.\\
\textsf{\bfseries Question:}& #3?\\\bottomrule
\end{tabular}
\end{center}
} \newcommand{\allModels}[1]{\mathbb{A}_{#1}}

\newif\ifhideproofs

\ifhideproofs
\usepackage{environ}
\NewEnviron{hide}{}

\fi

\usepackage[mathlines,switch]{lineno}
\newcommand*\linenomathpatch[1]{%
    \cspreto{#1}{\linenomath}%
    \cspreto{#1*}{\linenomath}%
    \csappto{end#1}{\endlinenomath}%
    \csappto{end#1*}{\endlinenomath}%
}
\newcommand*\linenomathpatchAMS[1]{%
    \cspreto{#1}{\linenomathAMS}%
    \cspreto{#1*}{\linenomathAMS}%
    \csappto{end#1}{\endlinenomath}%
    \csappto{end#1*}{\endlinenomath}%
}

\expandafter\ifx\linenomath\linenomathWithnumbers
\let\linenomathAMS\linenomathWithnumbers
\patchcmd\linenomathAMS{\advance\postdisplaypenalty\linenopenalty}{}{}{}
\else
\let\linenomathAMS\linenomathNonumbers
\fi

\linenomathpatch{equation}
\linenomathpatchAMS{gather}
\linenomathpatchAMS{multline}
\linenomathpatchAMS{align}
\linenomathpatchAMS{alignat}
\linenomathpatchAMS{flalign}
\newcommand{\ENT}{\problemFont{\textsc{Ent}}}
\newcommand{\pmCircOrderModelChecking}{\ENT(\CPLPref)}
\newcommand{\pmCircOrderModelCheckingTS}{\ENT(\PDLPref)}
\newcommand{\pmCircOrderModelCheckingTPL}{\ENT(\TPLPref)} %
\newcommand{\succinct}{\textsc{Succ}}
\newcommand{\PDL}{\logicFont{PDL}}

\renewcommand{\PL}{\ensuremath{\mathrm{PL}}\xspace}
\newcommand{\TPL}{\ensuremath{\logicFont{TPL}}\xspace}

\newcommand{\CPL}{\logicFont{CPL}\xspace}
\newcommand{\AC}[1]{\complClFont{AC}^{#1}}

\newcommand{\formulaOne}{\ensuremath{\varphi}}
\newcommand{\formulaTwo}{\ensuremath{\psi}}
\newcommand{\formulaThree}{\ensuremath{\gamma}}

\newcommand{\pmLogicPref}[1]{\ensuremath{{#1}\kern-0.5ex\raise0.95ex\hbox{\tiny\logicFont{pref}}}\xspace}
\newcommand{\TPLPref}{\pmLogicPref{\TPL}}
\newcommand{\CPLPref}{\pmLogicPref{\CPL}}
\newcommand{\PDLPref}{\pmLogicPref{\PDL}}

\DeclareMathOperator{\leqlogm}{\leq_{m}^{\mathrm{log}}}
\DeclareMathOperator{\lex}{<_{\mathrm{lex}}}
\DeclareMathOperator{\rlex}{>_{\mathrm{lex}}}
\newcommand{\complClFont}[1]{\mathsf{#1}}
\newcommand{\Ptime}{\complClFont{P}}
\newcommand{\OLMS}{\problemFont{OLMS}}

\DeclareMathOperator{\id}{\mathrm{id}}

\newcommand{\NC}[1]{\complClFont{NC}^{#1}}
\newcommand{\OLMSco}{\ensuremath{\overline{\problemFont{OLMS}}}}

\newcommand{\ssLogic}{\ensuremath{\mathscr{L}}}
\newcommand{\ssSystem}{\ensuremath{\mathbb{S}}}
\newcommand{\ssFormulas}{\ensuremath{\mathcal{L}}}
\newcommand{\ssInt}{\ensuremath{\Omega}}

\usepackage{todonotes}

\newcommand{\MYParagraph}[1]{\par\smallskip\noindent\textbf{#1}}

\makeatletter
\newcommand{\textlabelmarker}[1]{%
    \protected@edef\@currentlabel{#1}%
    \phantomsection%
}

\makeatother

\begin{document}
\thispagestyle{plain}

\maketitle

\begin{abstract}
This paper considers the complexity and properties of KLM-style preferential reasoning in the setting of propositional logic with team semantics and dependence atoms, also known as propositional dependence logic.
Preferential team-based reasoning is shown to be cumulative, yet violates System~P. We give intuitive conditions that fully characterise those cases where preferential propositional dependence logic satisfies System~P.  We show that these characterisations do, surprisingly, not carry over to preferential team-based propositional logic. Furthermore, we show how classical entailment and dependence logic entailment can be expressed in terms of non-trivial preferential models. 
Finally, we present the complexity of preferential team-based reasoning for two natural representations.
This includes novel complexity results for classical (non-team-based) preferential reasoning.
\end{abstract}

\section{Introduction}
\label{sec:introduction}

Preferential reasoning in style of Kraus, Lehmann and Magidor (\citeyear{KS_KrausLehmannMagidor1990})---henceforth abbreviated by KLM---is one of the main non-monotonic reasoning approaches that is well accepted in knowledge representation and reasoning, with connections to, e.g., belief change \cite{KS_MakinsonGaerdenfors1991} and human-like reasoning \cite{KS_RagniKernIsbernerBeierleSauerwald2020}; see also Gabbay et al. (\citeyear{KS_GabbayHoggerRobinson1993}) and Brewka et al. (\citeyear{KS_Brewka1997}) for a general placement within non-montonic reasoning.
The semantic core of KLM-style preferential reasoning is its very elegant construction by preferential models.
Roughly, a preferential model provides a strict partial order~$ \prec $ for a set of interpretations of some underlying logic (which is often classical propositional logic). 
Then, one says a formula \emph{$ \formulaTwo $ is preferentially entailed  from $ \formulaOne $} if all $\prec$-minimal models of $ \formulaOne $ are models of $ \formulaTwo $, i.e.,
\begin{center}
    \( \formulaOne \nmableit \formulaTwo \) if \( \minOf{\modelsOf{\formulaOne}}{\prec} \subseteq \modelsOf{\formulaTwo} \).
\end{center}
Intuitively, when a non-monotonic inference \( \formulaOne \nmableit \formulaTwo \) 
is generically understood
as \enquote{when \( \formulaOne \) holds, then usually \( \formulaTwo \) holds}, the preferential reasoning reading of \emph{\enquote{usually}} is \emph{\enquote{one expects that}} \cite{KS_GaerdenforsMakinson1994}. 
Hence, the intuition is that $\prec$ expresses a degree of exceptionality on the interpretations, i.e., the more preferred interpretations are less exceptional.
Another feature of preferential reasoning is that it is exactly characterized by the System~P postulates when the underlying logic is classical (KLM, \citeyear{KS_KrausLehmannMagidor1990}).
Because the System~P postulates are so widely accepted, preferential reasoning is sometimes considered as the \enquote{conservative core of non-monotonic reasoning} \cite{KS_Pearl1989,KS_Gabbay1984}.
Team semantics is a logical framework for studying concepts and phenomena that arise in the presence of plurality of objects. These concepts include, e.g.,   functional dependence ubiquitous in database theory and conditional independence of random variables in statistics. %
 The start of the field of team semantics can be traced back to the introduction of (first-order) dependence logic by Väänänen in \cite{vaananen07}.
In dependence logic, formulas are interpreted by sets of assignments (teams) instead of single assignments as in the usual classical semantics.
Syntactically, dependence logic 
introduces
new atomic formulas called dependence atoms $\dep{\vec{x},y}$ expressing that the values of the variables $\vec x$ functionally determine the value of the variable~$y$. %
During the past decade, the expressivity and complexity aspects of dependence logic and other team-based logics have been extensively studied and interesting connections have been found to areas such as database theory \cite{HannulaKV20,HannulaK16},  meta-finite model theory  \cite{abs-2003-00644}, inquisitive logic \cite{10.1215/00294527-2019-0033}, and epistemic logic \cite{Galliani15}. These works focus on logics in the first-order, propositional and modal team semantics, and more recently also in the multiset \cite{DurandHKMV18}, probabilistic \cite{HKMV18} and semiring settings \cite{BarlagHKPV23}. 

In this paper, we study preferential propositional dependence logic, i.e., preferential entailment with propositional dependence logic as underlying logic. A far as the authors know, a merger of logics in team semantics and non-monotonic reasoning has not been studied so far except for \cite{JY23}, where the former applies a certain (non-monotonic) team-based modal logic to the formal analysis of natural language.
In the following, we present the motivation for our study and then present an overview of this paper, including our main contributions.

\MYParagraph{Motivation.} 
Combining team-based reasoning and preferential reasoning is a promising way to obtain a novel conceptually rich family of reasoning approaches.
Consider, for instance, the classical example with \emph{birds} (\( b \)), \emph{flies} (\( f \)), and penguin (\( p \)). 
First, preferential entailment \( \dep{b,f} \nmableit \neg p \) reads technically as \enquote{all maximally preferred teams that satisfy \( \dep{b,f} \) also satisfy \( \neg p \)}. 
There is no obvious way to formulate the latter kind of expression in existing team-based logics, so injecting non-monotonicity is a valuable extension of team logic.
Note that  \enquote{\( \dep{b,f} \nmableit \neg p \)} does not imply that \( {=}\,(b,f) \land p  \) is inconsistent.
Then, when employing the typical understanding of preferential reasoning as realising inference by expectation, we obtain the following.
The dependence atom $\dep{b,f}$ expresses that whether it is a \emph{bird} fully determines whether it \emph{flies}. Thus, the (monotonic) entailment \( \dep{b,f} \models \neg p \) states that
\begin{center}
    \emph{\enquote{when whether it is a \emph{bird} (\( b \)) determines whether it \emph{flies}\emph{flies} (\( f \)), then it is not a penguin (\( \neg p \))}}
\end{center}
and the preferential entailment \( \dep{b,f} \nmableit \neg p \) reads as 
\begin{center}
    \emph{\enquote{when whether it is a \emph{bird} (\( b \)) determines whether it \emph{flies} (\( f \)), then one \textbf{expects} not a penguin (\( \neg p \))}}.
\end{center}
This is an expression that preferential reasoning with an underlying classical logic does not permit.
But we do not have to stop with this kind of understanding. A team corresponds to a plurality of objects, which permits various understandings of what a team stands for. Moreover, the preferential setting allows us to explore new understandings of the underlying order $\prec$.
Dependent on the application context, one reads \( \dep{b,f} \nmableit \neg p \), e.g., as follows:
\begin{itemize}
    \item\emph{Teams as databases:}
\enquote{When the value of \( b \) determines the value of \( f \) in a database, then one expects that the value of \( p \) is~\( 0 \).}

\item\emph{Teams as possible worlds.} \enquote{When the agent is convinced that whether \( f \) holds in a world always depends on \( b \), then usually the agent expects that \( p \) does not hold.}

\item\emph{Teams as answers to a question (inquisitive reading).} \enquote{One expects that \( p \) does not hold whenever in all answers \( f \) depends on \( b \).}

\item\emph{Teams as datasets and \( \prec \) orders them by reliability.} \enquote{In all most reliable datasets in which \( b \) determines the value of \( f \), it does not hold \( p \).} 
\end{itemize}
These are examples of interpretations of preferential dependence logic. 
We expect that one discovers more potential interpretations and applications of preferential team-based logics, when one considers preferential versions of other team-based logics.

\MYParagraph{Contributions.}  In this paper, we consider the complexity and properties of KLM-style preferential logics in the context of team-based logics. Specifically, we will encounter the preferential counterparts of the following logics:
\begin{itemize}
    \item Propositional logic with classical semantics  \hfill(\( \CPL \))
    \item Propositional logic with team-based semantics \hfill(\( \TPL \))
    \item Propositional dependence logic \hfill(\( \PDL \))
\end{itemize}
Our study will focus on preferential propositional dependence logic (\PDLPref).
But, we will also discuss preferential entailment of propositional logic with classical semantics (\CPLPref) and team-based semantics (\TPLPref).
The following list summarizes the main contribution of this paper:
\begin{itemize}
    \item{[\emph{Relationship of \PDLPref to System~P.}]} 
    It is shown that \PDLPref satisfies System~C and violates System~P. 
    We present two properties, \eqref{eq:StarProperty} and \eqref{eq:TriangleProperty} (see p.\pageref{eq:StarProperty}), for which each of them precisely characterize those preferential models in which System~P is satisfied.

\pagebreak[3]
    \item{[\emph{Properties of \TPLPref.}]} 
    We observe that characterization of System~P  via \eqref{eq:StarProperty} and \eqref{eq:TriangleProperty} does not carry over to \TPLPref from \PDLPref. This is surprisingly, as \TPL is a fragment of \PDL. 
    It is shown that \TPLPref still satisfies System~C and violates System~P.

    \item{[\emph{Complexity of Preferential Reasoning.}]} We give a full classification in terms of tractable and intractable cases for the problem of inference from a given preferential model for.
    Note that, unlike the problem of inference from set of conditional assertions~\cite{KS_LehmannMagidor1992,KS_EiterGottlob1992} for \CPLPref, complexity of inference from preferential mode for \PDLPref, \CPLPref and \TPLPref has not been studied before.
    We provide upper and lower bounds for the complexity of preferential classical propositional logic and preferential propositional dependence logic. 
    Table~\ref{tab:compl-overview} summarises the complexity results. 
\end{itemize}
\begin{table}[tb]
\centering
\begin{tabular}{@{}l@{}c@{}c@{\hskip0.7em}l@{\hskip0.3em}}
    \toprule
    \textbf{Problem}                             & \textbf{Tract.} &       \textbf{Complexity}       & \textbf{Result}                         \\ \midrule
    $\pmCircOrderModelChecking$                  &     \checkmark      &   $\in\Ptime$, $\NC{1}$-hard    & Thm.~\ref{thm:pent-pl}                  \\
    $\succinct\pmCircOrderModelChecking_{\rlex}$ &   \xmark        &      $\Delta^p_2$-complete      & Thm.~\ref{thm:COMC-deltap2-complete}    \\
    $\succinct\pmCircOrderModelChecking$         &   \xmark & $\in\Pi^p_2$, $\Delta^p_2$-hard & Thm.~\ref{thm:suc-pent-pl-upper-lower}  \\ %
    $\pmCircOrderModelCheckingTS$                &  \xmark  &   $\in\Theta_2^p$, $\NP$-hard   & Thm.~\ref{thm:pent-pdl}                 \\
    $\succinct\pmCircOrderModelCheckingTS$       &  \xmark   & $\in\Pi^p_2$, $\Delta^p_2$-hard & Thm.~\ref{thm:suc-pent-pdl-upper-lower} \\
    $\pmCircOrderModelCheckingTPL$               &     \checkmark      &   $\in\Ptime$, $\NC{1}$-hard    & Col.~\ref{cor:pent-tpl}                 \\
    $\succinct\pmCircOrderModelCheckingTPL$      &  \xmark   & $\in\Pi^p_2$, $\Delta^p_2$-hard & Col.~\ref{cor:pent-tpl}                 \\ \bottomrule
\end{tabular}
\caption{Overview of novel complexity results for entailment based on preferential models. 
    \textbf{Tract.} stands for tractability.}
\label{tab:compl-overview}
\end{table}
In the next section, we present the preliminaries on logic and computational complexity.
Section~\ref{sec:background_KLM_logic} presents the background on preferential reasoning. In Section~\ref{sec:axiomaticsPDL} we study the relationship of preferential propositional dependence logic to System~P.
A non-trivial preferential representation of standard entailment is presented in Section~\ref{sec:prefReconstruct}.
In Section~\ref{sec:discussion}, we discuss implications for preferential proposition logic with team semantics.
Section~\ref{sec:complexity} is dedicated to presenting upper and lower bounds for the complexity of preferential entailment.
Finally, Section~\ref{sec:conclusion} concludes the paper.

\section{Preliminaries}
\label{sec:background_team_based_logic}
We present the background on propositional logics with classical semantics, team semantics and propositional dependence logic %
(a survey on team-based logics can be found by \citeauthor{AKV16}, \citeyear{AKV16}).
Furthermore, we present the background on computational complexity.

\MYParagraph{Language of Propositional Logic.}
In the following, we describe some logics we consider in this paper.
We denote by $\Prop=\{\;p_i\mid i\in \mathbb{N}\;\}$ the countably infinite set of propositional variables. 
We will use letters $p,q,r,\dots$ (with or without subscripts) to stand for elements of $\Prop$. 
In this paper, we consider propositional formulas in negation normal form, i.e., 
well-formed \PL-formulas $\formulaOne$  are formed by the grammar:
 \[\formulaOne\Coloneqq p\mid \neg p \mid \bot\mid\top\mid \formulaOne\wedge\formulaOne\mid\formulaOne\vee\formulaOne,\]
where $p\in\Prop$, and $\top,\bot$ are the usual syntactic sugars for true and false. We write $\Prop(\formulaOne)$ for the set of propositional variables occurring in $\formulaOne$.

\MYParagraph{Classical Propositional Logic (\CPL).}
We consider the \emph{classical semantics} for \PL-formulas.
If one considers a non-empty finite subset $N\subseteq \Prop$ of propositional variables, then define for valuations $v\colon N\to\{0,1\}$ over \( N \) and \PL-formulas \( \formulaOne \):
\[ \llbracket \formulaOne\rrbracket^c \coloneqq\{v\colon N\to\{0,1\}\mid v\models\formulaOne\}. \]
We also write $v\models p$ in case $v(p)=1$, and $v\not\models p$ otherwise. 
The valuation function $v$ is extended to the set of all \PL-formulas in the usual way.
Let denote by $\allModels{N}$ the set of all assignments over $N$. 
Furthermore, we will write $\PL(N)$ for all propositional formulas using variables only from $N$.
%
	%
%
%
We write $\formulaOne\models^c\formulaTwo$ for $ \llbracket \formulaOne\rrbracket^c \subseteq \llbracket \formulaTwo\rrbracket^c $ and $\formulaOne\equiv^c\formulaTwo$ if both $\formulaOne\models^c\formulaTwo$ and $\formulaTwo\models^c\formulaOne$.
When we talk in this paper about \CPL, we refer to the logic with \PL-formulas and classical semantics. 

%

%
%
%

\MYParagraph{Propositional Logic with Team Semantics (\TPL).}
Next, we define \emph{team semantics} for \PL-formulas (cf. \cite{HannulaKVV15,YangV16}). A  team $X$ is a set of valuations for some finite set $N\subseteq \Prop$. We write $\dom(X)$ for the domain $N$ of $X$.%

\begin{definition}[Team semantics of \PL]
	Let $X$ be a team. For any \PL-formula $\formulaOne$ with $\dom(X)\supseteq \Prop(\formulaOne)$, the satisfaction relation $X\models\formulaOne$ is defined inductively as:
	\begin{align*}
		&X\models p&&\text{if for all }v\in X: v\models p\\
		&X\models\neg p&&\text{if for all } v\in X: v\not\models p\\
		&X\models \bot&&\text{if } X=\emptyset\\
		&X\models \top&&\text{is always the case}\\
		&X\models\formulaOne\wedge\formulaTwo&&\text{if } X\models\formulaOne \text{ and } X\models\formulaTwo\\
		&X\models\formulaOne\vee\formulaTwo&&\text{if there exist } Y,Z\subseteq X\\
		&&&\text{ s.t. }X=Y\cup Z, Y\models\formulaOne, \text{ and }Z\models\formulaTwo.
	\end{align*}
\end{definition}
\noindent The set of all teams \( X \) with \( X\models\formulaOne \) is denoted by \( \modelsOf{\formulaOne} \).
For any two \PL-formulas \( \formulaOne,\formulaTwo \), we write $\formulaOne\models^t\formulaTwo $ if \( \modelsOf{\formulaOne} \subseteq \modelsOf{\formulaTwo} \).
Write $\formulaOne\equiv^t\formulaTwo$ if both $\formulaOne\models^t\formulaTwo$ and $\formulaTwo\models^t\formulaOne$.
With \TPL we refer to the logical setting of \PL-formulas with team semantics.
We define the following properties for a formula~$\formulaOne$:
\begin{itemize}
		\item $X\models \formulaOne \iff \text{for all } v\in X,~\{v\}\models\formulaOne$. \hfill {\small(\textbf{Flatness})}
		\item $\emptyset \models \formulaOne$.\hfill{\small(\textbf{Empty team property})}
		\item If $X \,{\models}\, \formulaOne$ and $Y\,{\subseteq}\, X$, then ${Y\models \formulaOne}$.\hfill{\small(\textbf{Downwards closure})}
\end{itemize}
\begin{proposition}\label{prop:pdl_pincl_properties}
	\TPL has the flatness property, empty team property and downwards closure property. 
\end{proposition}

Due to the flatness property, logical entailment of propositional logic with team-based semantics $\models^t$ and logical entailment of propositional logic with classical semantics $\models^c$ coincide.
However, we will see later that these different semantic approaches will lead to different preferential entailment relations.

\MYParagraph{Propositional Dependence Logic.}
\label{sec:pdl_pincl}
A {\em (propositional) dependence atom} is a string $\dep{a_1\dots a_k,b}$, 
in which $a_1,\dots,a_k,b$ are propositional variables from \Prop. 
The  team semantics of dependence atoms is defined  as follows, whereby \( \vec{a} \) stands for \( a_1,\dots, a_k \):
\begin{align*}
	& X \models \dep{\vec{a},b} && \text{if}~\text{for all }v, v' \in X, \\
	&&& \phantom{if} v(\vec{a})=v'(\vec{a}) \textrm{ implies }v(p)=v'(p).
\end{align*}
A dependence atom with the empty sequence in the first component will be abbreviated as $\dep{p}$ and called {\em constancy atoms}. The team semantics of the constancy atoms is reduced to
\begin{equation*}
	X \models \dep{p} ~~\text{if}~~ \text{for all }v, v' \in X, v(p)=v'(p).
\end{equation*}
We define the language of {\em propositional dependence logic} (denoted as \pdl) as the extension of \PL-formulas with dependence atoms.
With \PDL we refer to the whole logical approach, including the language and the above-mentioned semantics.
We consider an example for \PDL.
\begin{example}\label{example_dep_atm_propositional}
	Consider the team $X$ over $\{p,q,r\}$ defined by:
	\begin{center}
		\begin{tabular}{cccc}\toprule
			&$p$&$q$&$r$\\\midrule
			$v_1$&$1$&$0$&$0$\\
			$v_2$&$0$&$1$&$0$\\
			$v_3$&$0$&$1$&$0$\\\bottomrule
		\end{tabular}
	\end{center}
	We have $X\models\dep{p,q}$ and $X\models\dep{r}$. Moreover,  $X\models\dep{p}\vee \dep{p}$ but  $X\not \models\dep{p}$. %
\end{example}

The following proposition describes properties of \PDL.
\begin{proposition}\label{prop:pdl_pincl_properties}
	\PDL has the empty team property and downwards closure property. 
\end{proposition}
Note that \PDL does not satisfy the flatness property. 
We can define the flattening $\phi^f$ of a  \pdl-formula by replacing all dependence atoms in $\phi$ by $\top$. Clearly, $\phi^f$ is a \PL formula. Furthermore, one checks easily that $\phi \models^t \phi^f$
and that \( \{v\}\models \phi \Leftrightarrow v\models \phi^f \) for all assignments $v$.

\MYParagraph{Generic View on Logics.} 
Some parts of this paper will require a generic view on logics. 
The following provides such a view that provides the right abstraction to capture those aspects of logics that we need in this paper in a generic way.
A satisfaction system is a triple \( \ssSystem = \tuple{\ssFormulas,\ssInt,\models} \), where \( \ssFormulas \) is the set of formulas, \( \ssInt \) is the set of interpretations, and  \( {\models} \subseteq \ssInt \times \ssFormulas \) is the model-relation.
An entailment relation for a satisfaction system is a relation \( {\nmableit} \subseteq \ssFormulas \times \ssFormulas\). 
A satisfaction system \( \ssSystem \) together with an entailment relation \( \nmableit \) is denoted as logic
 \( \ssLogic = \tuple{\ssFormulas,\ssInt,\models,\nmableit} \).
We say a logic \( \ssLogic \) is \emph{standard}, if \( \nmableit \) is the canonical logical entailment given by \( \formulaOne \nmableit  \formulaTwo \) if \( \modelsOf{\formulaOne} \subseteq \modelsOf{\formulaTwo} \), whereby \( \modelsOf{\formulaOne} = \{ v \in \ssInt \mid v \models \formulaOne  \}  \). The canonical logical entailment is often written with the same symbol $\models$.
The propositional logics described in this section provide the following instances of the generic approach for each \( N \subseteq \Prop \), which are all standard logics:
\begin{center}
    \begin{tabular}{@{}l@{\hskip0.05em}l@{\hskip0.05em}l@{}}
    \( \CPL_{N} \)  & \( =\! \tuple{\ssFormulas^{\CPL}_{N},\ssInt^{\CPL}_{N},\models,\models^{\CPL}}  \) & \( {=}\, \tuple{\PL(N),\allModels{N},\models,\models^c} \)\\
    \( \TPL_{N} \) & \( =\! \tuple{\ssFormulas^{\TPL}_{N},\ssInt^{\TPL}_{N},\models,\models^{\TPL}}\) & \( {=}\, \tuple{\PL(N),\mathbb{T}_{N},\models,\models^t} \)\\
    \( \PDL_{N} \) & \( =\! \tuple{\ssFormulas^{\PDL}_{N},\ssInt^{\PDL}_{N},\models,\models^{\PDL}} \) & \( {=}\, \tuple{\pdl, \mathbb{T}_{N},\models,\models^t}\)
\end{tabular}
\end{center}
Often, we will not mention \( N \) explicitly and assume that there is some \( N \) of appropriate size.
Moreover, we will write $\models$ instead of $\models^t$ when there is no ambiguity.

\MYParagraph{Computational Complexity.}
We assume basic familiarity with computational complexity theory~\cite{KS_Papadimitriou1994}. 
We will make use of the complexity classes $\Ptime$ and $\Delta_2^p=\Ptime^{\NP}$, as well as standard reducibility notions, e.g., logspace-many-to-one reductions~${\leqlogm}$.

Before we define a natural complete problem for $\Delta_2^p$, we need to formally introduce the lexicographic order on assignments. 
Let $P\subsetneq\Prop$ with $|P|=n$. 
An assignment function $v\colon P\to\{0,1\}$ is interpreted as a string $v_s$ from $\{0,1\}^n$ in the natural way, i.e., $v_s=v(x_1)\cdots v(x_n)$ if $P=\{x_1,\dots,x_n\}$.  
We denote by $v_s[i]$ for $1\leq i \leq n$ the $i$th bit of $v_s$. 

\begin{definition}[lexicographic order]
    For two strings $s,s'\in\{0,1\}^n$ with $s\neq s'$ we say that $s\lex s'$ if there exists a (possibly empty) prefix $0\leq j\leq n$ such that $s[k]=s'[k]$ for all $0\leq k\leq j$ and $s[k+1]<s'[k+1]$. 
\end{definition}

\problemdef{\OLMS}{A propositional formula $\formulaOne$ over variables $\{x_1,\dots,x_n\}$}{$\formulaOne$ is satisfiable, and for the largest satisfying assignment $\theta$ with respect to $\lex$ do we have that $\theta(x_n)=1$}

\begin{proposition}[\citeauthor{KS_Krentel1988}~\citeyear{KS_Krentel1988}]
    $\OLMS$ is $\Delta_2^p$-complete under $\leqlogm$-reductions.
\end{proposition}

The complementary problem to \OLMS is denoted here as \OLMSco, where $\theta(x_n)\neq1$.
As \( \Delta_2^p \) is a deterministic complexity class, \( \OLMSco \) of \( \OLMS \) is also $\Delta_2^p$-complete.
Now we will present the standard notion of a circuit family. 
For a comprehensive overview of the topic of circuit complexity, refer to the textbook by Vollmer~(\citeyear{v99}).

\begin{definition}
    Let $B=\{\lor,\land,\lnot,0,1\}$. 
    A \emph{circuit family} over $B$ is a sequence $\mathcal C=(C_0,C_1,\dots)$, where for every $n\in\N$, $C_n$ is a circuit over $B$ with $n$ inputs. 
    Let $f^n$ be the function computed by $C_n$. 
    Then we say that $\mathcal C$ computes the function $f\colon \{0,1\}^*\to \{0,1\}^*$, defined for every $w\in\{0,1\}^*$ by $f(w)\coloneqq f^{|w|}(w)$. 
    We write $f=(f^n)_{n\in\N}$ and $\mathcal C=(C_n)_{n\in\N}$.
\end{definition}

As usual $\AC0$ denotes the class of all polynomial-sized circuit families of constant depth using gates with unbounded fan-in. 
$\NC{1}$ denotes the class of all polynomial-sized circuit families of logarithmic depth using gates with fan-in two.

\begin{example}\label{ex:lex-circ}
    Given two binary strings $\bar a=a_{n-1}\cdots a_{0},\bar b=b_{n-1}\cdots b_{0}$ (so the least significant bit is the rightmost bit), the lexicographic order $\lex$ can be defined via an $\AC0$-circuit, where the overline $\overline{\,\cdot\,}$ means an outermost negation:
    \[
    \overline{\bigvee_{i=0}^{n-1} \big(a_i > b_i  \land \bigwedge_{j=i+1}^{n-1} (a_j = b_j)\big)},
    \]
    and conjunctions (disjunctions) over the empty set are defined as true (false). 
    Here $a>b$ is then merely encoded via $a\land\lnot b$, and $a=b$ with $(a\land b)\lor(\lnot a\land\lnot b)$.
\end{example} 

\section{Background on Preferential Logics}
\label{sec:background_KLM_logic}
In this section, we present background on preferential logics in the style of Kraus, Lehmann and Magidor (\citeyear{KS_KrausLehmannMagidor1990}) (KLM).
In preferential logic, an entailment  \( \formulaOne \nmableit \formulaTwo \) holds, when minimal models of \( \formulaOne \) are models of \( \formulaTwo \). This is formalized via preferential models, which we introduce in the following.

	For a  strict partial order \( {\pmP} \subseteq \mathcal{S} \times \mathcal{S} \) on a set \( \mathcal{S} \) and a subset \( S \subseteq \mathcal{S} \), an element \( s \in S \) is called \emph{minimal in \( S \) with respect to \( \pmP \)} if for each \( s' \in S \) holds  \( s' \not \pmP s \). 
	Then, \( \minOf{S}{{\pmP}} \)  is the set of all \( s \in S \)  that are minimal in \( S \) with respect to \( \pmP \).
\begin{definition}[KLM, \citeyear{KS_KrausLehmannMagidor1990}]
    Let \( \ssLogic = \tuple{\ssFormulas,\ssInt,\models,\models^{\ssLogic}}  \) be a logic.
    A \emph{preferential model} for \(  \ssLogic  \) is a triple \( \pmW=\tuple{\pmS,\pmL,\pmP} \) where \( \pmS \) is a set, \( \ell\colon \pmS \to \Omega \), \( \pmP \) is a strict partial order on \( \pmS \), and the following condition is satisfied:
    \begin{itemize}
        \item[]\emph{[Smoothness]} \( \statesOf{\formulaOne}=\{ s \in \pmS \mid \ell(s) \models \formulaOne \} \) is smooth with respect to \( \pmP \) for every formula \( \formulaOne \in \ssFormulas \), i.e, for each \( s \in \statesOf{\formulaOne} \) holds
        \( s\in \minOf{\statesOf{\formulaOne}}{\pmS} \) or there exists an \( s' \in \statesOf{\formulaOne} \) with \( s'\in \minOf{\statesOf{\formulaOne}}{\pmS} \) and \( s' \pmP s \).
    \end{itemize}
We say a preferential model is finite if $\pmS$ is finite.
\end{definition}
Smoothness guarantees the existence of minimal elements.
For convenience, we make use of the following abbreviations: \( \minOf{\modelsOf{A}}{\pmP} = \{ \pmL(s) \mid s\in \minOf{S(A)}{\pmP}  \} \) and we write \( s \models \formulaOne \) for \( \pmL(s) \models \formulaOne \).
\begin{definition}[KLM, \citeyear{KS_KrausLehmannMagidor1990}]
    Let \( \ssLogic = \tuple{\ssFormulas,\ssInt,\models,\models^{\ssLogic}}  \) be a standard logic.
    The entailment relation \( {\nmableitW} \subseteq \ssFormulas \times \ssFormulas \) for a preferential model \( \pmW \) for \(  \ssLogic \) is given by
    \begin{equation*}
    	\formulaOne \nmableitW \formulaTwo \text{ if }  \minOf{\modelsOf{\formulaOne}}{\pmP} \subseteq \modelsOf{\formulaTwo}\ .
    \end{equation*}
    An entailment relation \( {\nmableit} \subseteq \ssFormulas \times \ssFormulas \) is called \emph{preferential} if there is a preferential model \( \pmW \) for \( \mathscr{L} \) such that  \( {\nmableit} =   {\nmableitW}  \). 
\end{definition}%

Every preferential model \( \pmW \) for a standard logic \( \ssLogic = \tuple{\ssFormulas,\ssInt,\models,\models^{\ssLogic}} \) give rise to a logic \( \ssLogic_{\pmW} = \tuple{\ssFormulas,\ssInt,\models,{\nmableitW}} \). We say that \( \ssLogic_{\pmW} \) is a preferential logic for \( \ssLogic \). 
\emph{Preferential propositional dependence logic} (\PDLPref) refers to the preferential logics for \( \PDL \); and analogously for \emph{preferential classical propositional logic} (\CPLPref) and  \emph{preferential propositional logic with team semantics} (\TPLPref).

\section{On the Relationship of {\PDLPref} to System~P}
\label{sec:axiomaticsPDL}
In this section, we consider the satisfaction of properties by preferential propositional dependence logic.
We make use of the following postulates for non-monotonic entailment \( \nmableit \):
\begin{center}
    \vspace{-1em}
    \begin{minipage}[b]{0.51\linewidth}
        \begin{align}
            &\frac{}{\formulaOne\nmableit\formulaOne} \tag{Ref}\label{pstl:Ref}\\[0.25em]
            &\frac{\formulaOne\equiv\formulaTwo\hspace{0.5cm}\formulaOne\nmableit\formulaThree}{\formulaTwo\nmableit\formulaThree} \tag{LLE}\label{pstl:LLE}\\[0.25em]
            &\frac{\formulaOne\land\formulaTwo\nmableit\formulaThree\hspace{0.5cm}\formulaOne\nmableit\formulaTwo}{\formulaOne\nmableit\formulaThree} \tag{Cut}\label{pstl:Cut}
        \end{align}
    \end{minipage}\begin{minipage}[b]{0.45\linewidth}
        \begin{align}
            &\frac{\formulaOne\models\formulaTwo\hspace{0.5cm}\formulaThree\nmableit\formulaOne}{\formulaThree\nmableit\formulaTwo} \tag{RW}\label{pstl:RW}\\[0.25em]
            &\frac{\formulaOne\nmableit\formulaTwo\hspace{0.5cm}\formulaOne\nmableit\formulaThree}{\formulaOne\land\formulaTwo\nmableit\formulaThree} \tag{CM}\label{pstl:CM}\\[0.25em]
            &\frac{\formulaOne\nmableit\formulaThree\hspace{0.5cm}\formulaTwo\nmableit\formulaThree}{\formulaOne\lor\formulaTwo\nmableit\formulaThree} \tag{Or}\label{pstl:Or}
        \end{align}
    \end{minipage}
\end{center}
Note that \( \models \) is the entailment relation of the underlying monotonic logic, and \( \equiv \) the respective semantic equivalence.
The rules \eqref{pstl:Ref}, \eqref{pstl:RW}, \eqref{pstl:LLE},  \eqref{pstl:CM} and  \eqref{pstl:Cut} forming \emph{System~C}.
Notably, the rule \eqref{pstl:CM} goes back to the foundational paper on non-monotonic reasoning system by Gabbay (\citeyear{KS_Gabbay1984}) and is a basic wakening of monotonicity. 
\emph{System~P}  consists of all rules of System~C and the rule \eqref{pstl:Or}. The rule of \eqref{pstl:Or} is motivated by reasoning by case \cite{KS_Pearl1989}.

A seminal result for {\CPL} by KLM~(\citeyear{KS_KrausLehmannMagidor1990}) is the direct correspondence between preferential entailment and  System~P.
\begin{proposition}[KLM, \citeyear{KS_KrausLehmannMagidor1990}]\label{prop:KLM_preferential_representation}
    A entailment relation \( {\nmableit} \) for \CPL satisfies System~P if and only if \( {\nmableit} \) is preferential.
\end{proposition}

Systems~C and P are of utmost importance for non-monotonic reasoning. System~C is considered to be the basic properties of good non-monotonic reasoning, and
System~P is considered the \enquote{conservative core} of non-monotonic reasoning~(KLM, \citeyear{KS_KrausLehmannMagidor1990}).

\MYParagraph{Satisfaction  of \PDLPref of System~C.}
Our first observation is that preferential entailment for \PDL satisfies System~C. This is a novel result as the proof of System~C satisfaction for \CPLPref given by KLM does not carry over to preferential propositional dependence logic.
\begin{proposition}\label{prop:systemCteams}
    \PDLPref satisfies System~C.
\end{proposition}
\begin{proof}
   We show that \( {\nmableitW} \) satisfies all rules of System~C:\\
       \noindent[\emph{\ref{pstl:Ref}.}]
       Considering the definition of \( \nmableitW \) yields that \( \formulaOne \nmableitW \formulaOne \) if for all minimal \( s\in\statesOf{\formulaOne} \) holds \( \ell(s)\models\formulaOne \). By the definition of \( \statesOf{\formulaOne} \), we have \( s\in\statesOf{\formulaOne} \) if \( \ell(s)\models\formulaOne \). Consequently, we have \( \formulaOne\nmableitW \formulaOne \).

       \noindent[\emph{\ref{pstl:LLE}.}] 
       From \( \formulaOne\equiv\formulaTwo \), we obtain that \( \statesOf{\formulaOne}=\statesOf{\formulaTwo} \) holds.
       By using this last observation and the definition of \( \nmableitW \), we obtain \( \formulaTwo\nmableitW\formulaThree \) from \( \formulaOne\nmableitW\formulaThree \).

       \noindent[\emph{\ref{pstl:RW}.}] 
       Clearly, by definition of \( \formulaOne\models\formulaTwo \) we have \( \modelsOf{\formulaOne}\subseteq\modelsOf{\formulaTwo} \).
       From the definition of \( \formulaThree\nmableitW\formulaOne \), we obtain that \( \ell(s)\models\formulaOne \) holds for each minimal \( s\in\statesOf{\formulaThree} \).
       The condition \( \ell(s)\models\formulaOne \) in the last statement is equivalent to stating \( \ell(s)\in\modelsOf{\formulaOne} \). Because of \( \modelsOf{\formulaOne}\subseteq\modelsOf{\formulaTwo} \), we also have \( \ell(s)\in\modelsOf{\formulaTwo} \); and hence, \( \ell(s)\models\formulaTwo \)  for each minimal  \( s\in\statesOf{\formulaThree} \).
       This shows that \( \formulaThree\nmableitW\formulaTwo \) holds.

       \noindent[\emph{\ref{pstl:Cut}.}] 		
       By unfolding the definition of \( \nmableitW \), we obtain 
       \(  {\minOf{\statesOf{\formulaOne\land\formulaTwo}}{\prec}}  \subseteq \statesOf{\formulaThree} \) 
       from \( \formulaOne\land\formulaTwo \nmableitW \formulaThree \).
       Analogously, \( \formulaOne \nmableitW \formulaTwo \) unfolds to 
       \( \minOf{\statesOf{\formulaOne}}{\prec} \subseteq \statesOf{\formulaTwo} \).
       Moreover, employing basic set theory yields that \( \statesOf{\formulaOne\land\formulaTwo}=\statesOf{\formulaOne}\cap\statesOf{\formulaTwo}\subseteq\statesOf{\formulaOne} \) holds.
       From \( \statesOf{\formulaOne\land\formulaTwo}\subseteq\statesOf{\formulaOne} \) and \( \minOf{\statesOf{\formulaOne}}{\prec} \subseteq \statesOf{\formulaTwo} \), we obtain \( \minOf{\statesOf{\formulaOne}}{\prec} \subseteq \statesOf{\formulaOne\land\formulaTwo} \).
       Consequently, we also have that \( \minOf{\statesOf{\formulaOne}}{\prec}=\minOf{\statesOf{\formulaOne\land\formulaTwo}}{\prec} \) holds.
       Using the last observation and \( \minOf{\statesOf{\formulaOne\land\formulaTwo}}{\prec}  \subseteq \statesOf{\formulaThree} \), we obtain \( \minOf{\statesOf{\formulaOne}}{\prec}  \subseteq \statesOf{\formulaThree} \). Hence also \( \formulaOne \nmableitW \formulaThree \) holds.

       \noindent[\emph{\ref{pstl:CM}.}] 		
       By unfolding the definition of \( \nmableitW \), we obtain \(  \minOf{\statesOf{\formulaOne}}{\prec}  \subseteq \statesOf{\formulaTwo} \)  and \(  \minOf{\statesOf{\formulaOne}}{\prec}  \subseteq \statesOf{\formulaThree} \).
       We have to show that \( \minOf{\statesOf{\formulaOne\land\formulaTwo}}{\prec} \subseteq S(\formulaThree) \) holds.
       Let \( s \) be element of \( \minOf{\statesOf{\formulaOne\land\formulaTwo}}{\prec} \).
       Clearly, we have that \( s \in S(\formulaOne)  \) holds. We show by contradiction that \( s \) is minimal in \( S(\formulaOne) \). 
       Assume that \( s \) is not minimal in \( S(\formulaOne) \).
       From the smoothness condition, we obtain that there is an \( s'\in S(\formulaOne) \) such that \( s' \prec s \) and \( s' \) is minimal in \( S(\formulaOne) \) with respect to \( \prec \). 
       Because \( s' \) is minimal and because we have \( \minOf{\statesOf{\formulaOne}}{\prec}  \subseteq \statesOf{\formulaTwo} \), we also have that \(  s'\in S(\formulaTwo) \) holds and hence that \( s'\in S(\formulaOne\land\formulaTwo) \) holds. The latter contradicts the minimality of \( s \) in \( S(\formulaOne\land\formulaTwo) \).
       Consequently, we have that \( s \in \minOf{\statesOf{\formulaOne}}{\prec} \) holds. 
       Because we have \(  \minOf{\statesOf{\formulaOne}}{\prec}  \subseteq \statesOf{\formulaThree} \), we obtain \( \formulaOne \land \formulaTwo \nmableitW \formulaThree \).\qedhere
\end{proof}

\pagebreak[3]
\MYParagraph{Relationship of \PDLPref to System~P.}
The following example shows that Proposition~\ref{prop:KLM_preferential_representation} does not carry over to \PDL, i.e., there are preferential entailments for \PDL that witness a violation of \eqref{pstl:Or}. 
\begin{example}\label{ex:violateOR}
    Assume that \( N=\{p,q\}\subseteq\Prop \) holds.
    The following valuations \( v_1,v_2 ,v_3 \) will be important: %
    \begin{align*}
        v_1(p) &=v_1(q)=v_2(q)=1 & v_2(p)&= v_3(p)=v_3(q)=0
    \end{align*}
    We consider the teams \( X_{pq}=\{ v_1\} \), \( X_{\overline{p}{q}}=\{ v_2\} \), and  \( X_{p\leftrightarrow q}=\{ v_1,v_3\} \).
    Let \( \pmWpq = \tuple{\pmSpq,\pmLpq,\pmPpq} \) be the preferential model such that
    \begin{align*}
        \pmSpq & = \{ s_{\Int} \mid \Int \text{ is a  non-empty team} \} & \pmLpq(s_X) & = X         
    \end{align*}
    holds, and such that \( \pmPpq \) is the strict partial order given by (for the sake of readability we identify \( s_X \) with \( X \)).
    \begin{align*}
        X_{p\leftrightarrow q} & \pmPpq  X_{pq}             & X_{pq}              & \pmPpq X \\
        X_{p\leftrightarrow q} & \pmPpq X_{\overline{p}{q}} & X_{\overline{p}{q}} & \pmPpq X
    \end{align*}
    where \( X \) stands for every team different from \( X_{\overline{p}{q}} \) and  \(  X_{p\leftrightarrow q} \).
    We     obtain the following preferential entailments:
    \begin{align*}
        p & \nmableitWparam{\pmWpq} q &  \neg p & \nmableitWparam{\pmWpq} q & p\lor\neg p &\notnmableitWparam{\!\!\pmWpq} q
    \end{align*}
    This shows that \( \nmableitWparam{\pmWpq} \) violates \eqref{pstl:Or}, and thus, System~P.
\end{example}
\begin{proposition}
\PDLPref violates System~P.
\end{proposition}

We will now identify two properties that will fully capture those preferential entailment relations that satisfy all rules of System~P within \PDL. 
As first, we can observe that one obtains System~P, when the underlying preferential model \( \pmW =\tuple{\pmS,\pmL,\pmP} \) satisfies the following property:
\begin{equation}
    \minOf{\modelsOf{A \lor \formulaTwo}}{\pmP} \subseteq \minOf{\modelsOf{A}}{\pmP} \cup \minOf{\modelsOf{\formulaTwo}}{\pmP}	\tag{\( \star \)}\label{eq:StarProperty}
\end{equation}
One checks easily that \eqref{eq:StarProperty} guarantees satisfaction of \eqref{pstl:Or}.
\begin{proposition}\label{prop:SystemPifStar}
    Let \( \pmW \) be a preferential model for \PDL.
    If \eqref{eq:StarProperty} is satisfied for all formulas \( \formulaOne,\formulaTwo \), then \( \nmableitW \) satisfies \eqref{pstl:Or}.
\end{proposition}
\noindent Second, we say that \( \pmW \) satisfies the \eqref{eq:TriangleProperty}-property if for all state \( s\in\pmS \) hold:
\begin{equation}
\text{If\,} |\pmL(\!s\!)|\,{>}\,1 \text{,\,then\,} \pmL(\!s') \,{\subsetneq}\, \pmL(\!s\!) \text{\,and\,} s'\,{\pmP}\, s \text{\,for\,some\,} s'{\in\,}\pmS . \tag{\( \triangle \)}\label{eq:TriangleProperty}
\end{equation}
The \eqref{eq:TriangleProperty}-property demands (when understanding states as teams) that for each non-singleton team \( X \) exists a proper subteam \( Y \) of \( X \) that is preferred im \( \pmW \) over \( X \).

In the following theorem, we show that the  \( \triangle \)-property and \eqref{eq:StarProperty}-property guarantee satisfaction of System~P.
\begin{theorem}\label{main}
    Let \( \pmW=\tuple{\pmS,\pmL,\pmP} \)  be a preferential model for \PDL.
    The following statements are equivalent:
    \begin{enumerate}[(i)]
        \item \( \nmableitW \) satisfies System~P.
        \item \( \pmW \) satisfies the \( \triangle \)-property.
        \item The \eqref{eq:StarProperty}-property holds for all \( \formulaOne,\formulaTwo\in\PDL \).
    \end{enumerate}
\end{theorem}
\pagebreak[3]
We prepare the proof of Theorem~\ref{main} via the following lemmata.
For the first lemma, assume that $N=\{p_1,\dots,p_n\}$, and let $X$ be an $N$-team. 
We define the following formula:
\[\Theta_X:=\bigvee_{v\in X}(p_1^{v(1)}\wedge\dots\wedge p_n^{v(n)}) \ ,\]
whereby \( p_i^{v(i)} \) is \( p_i \) if \( v(p_i)=1 \) and \( \neg  p_i \) otherwise.
Note that for a team \( X \), the powerset \( \mathcal{P}(X) \) is the set of all subteams of \( X \).
It is straightforward to check the following.
\begin{lemma}\label{lem:definedownsets}
     $\Theta_X$ defines the family of subteams of $X$, i.e., we have that 
     \( \modelsOf{\Theta_X} = \mathcal{P}(X) \).
\end{lemma}

The next lemma guarantees that for a sufficient large enough teams \( X \) exist formulas \( \formulaOne,\formulaTwo \) such that \( X \) is a model of the disjunction \( \formulaOne \lor \formulaTwo \), but \( X \) is not a model of \( \formulaOne \) and \( \formulaTwo \).
\begin{lemma}\label{lem:dagger}
    For each team \( X \) with \( |X|>1 \) exists formulas \( \formulaOne \) and \( \formulaTwo \) such that
    \begin{align*}
        X \models \formulaOne \lor \formulaTwo \text{ , } X \not\models \formulaOne \text{ , and } X \not\models \formulaTwo.
    \end{align*}
\end{lemma}
\begin{proof}
    Since \( |X|>1 \) %
    there exists non-empty \( Y,Z\subseteq X \) such that \( X=Y\cup Z \) and \( Y\neq X \) and \( Z \neq X \).
    There are formulas \( \formulaOne \) and \(\formulaTwo \) such that \( \modelsOf{\formulaOne}=\mathcal{P}(Y) \) and \( \modelsOf{\formulaTwo}=\mathcal{P}(Z) \), namely \( \formulaOne=\Theta_Y \) and \( \formulaTwo=\Theta_Z \) from by Lemma~\ref{lem:definedownsets}.
\end{proof}

We will now show that the \eqref{eq:TriangleProperty}-property and the \eqref{eq:StarProperty}-property  describe the same preferential models.
\begin{lemma}\label{lem:triangle=star}
    A preferential model \( \pmW \) for \PDL satisfies the \eqref{eq:StarProperty}-property if and only if \( \pmW \) satisfies the  \eqref{eq:TriangleProperty}-property.
\end{lemma}
\begin{proof}
    Assume \eqref{eq:TriangleProperty} holds. Then it is easy to see that the minimal elements of the order $\prec$ are states that are mapped, via \( \pmL \), to singleton teams. Furthermore, by the downward closure property, for any $\formulaOne\vee \formulaTwo$ the minimal teams satisfying the formula are all singletons. Since for singleton teams the interpretation of $\vee$ is equivalent with that of the Boolean disjunction the property  (\( \star \)) follows.
    
    For the converse, assume that  (\( \star \)) holds and let \( X \) be a team with \( |X| > 1 \). We will show that  then there is some team \( Y  \) with
        \( Y\subsetneq X  \text{, } Y \neq \emptyset  \text{, and } Y \prec X \).
    Because \( X \) contains at least two valuations, there exist \( Y,Z\subseteq X \) such that \( X=Y\cup Z \) and \( Y\neq X \) and \( Z \neq X \).
    By (the proof of) Lemma~\ref{lem:dagger} there are formulas \( \formulaOne=\Theta_Y  \) and \( \formulaTwo=\Theta_Z  \) such that \( X \models \formulaOne \lor \formulaTwo \), yet \( X \not\models \formulaOne \) and \( X \not\models \formulaTwo \).
    Using this and \eqref{eq:StarProperty}, we obtain that \( X \notin \minOf{\formulaOne \lor \formulaTwo}{\prec} \) holds.
    However, by smoothness of \( \prec \), the set $\modelsOf{\formulaOne \lor \formulaTwo}=\mathcal{P}(X)$ contains a team \( X' \) such that \( X' \prec X \). Now $X'$ is a witness for the \eqref{eq:TriangleProperty}-Property.
\end{proof}

Now we are ready to give the proof of Theorem \ref{main}.
\begin{proof}[Proof of Theorem~\ref{main}]
    By Lemma~\ref{lem:triangle=star}, it suffices to show \eqref{eq:StarProperty}\( \Rightarrow \)\eqref{pstl:Or} and {\eqref{pstl:Or}\( \Rightarrow \)\eqref{eq:TriangleProperty}.}    
    From Proposition~\ref{prop:SystemPifStar} we obtain \eqref{eq:StarProperty}\( \Rightarrow \)\eqref{pstl:Or}.       
    It remains to show \eqref{pstl:Or}\( \Rightarrow \)\eqref{eq:TriangleProperty}.
        Towards a contradiction, assume that  \eqref{eq:TriangleProperty} fails.  Then there exists a team $X$ of size $j\ge 2$ such that for all $Y\subseteq X$, $Y \not \prec X$. Let $j=l+k$ ($l,k\ge 1 $ and $l\le k$) and define
        $$\formulaOne:= \Theta_X\wedge (\theta\vee \cdots \vee \theta ),$$
        where $\theta :=\bigwedge_{1\le i\le n}\dep{p_i} $ and $\formulaOne$ has $l$ many copies of $\theta$. It is easy to check that $\formulaOne$ is satisfied by subteams of $X$ of cardinality at most $l$. The formula $\formulaTwo$ is defined similarly with $k$ copies of $\theta$ in the disjuncts. Now it holds that $\formulaTwo \models \formulaTwo$,   $\formulaOne \models \formulaTwo$ but $X\not \models \formulaOne, \formulaTwo $. Using reflexivity and right weakening, it follows that 
        $ \formulaTwo\nmableitW \formulaTwo$ and  $\formulaOne\nmableitW \formulaTwo$. On the other hand,  
        since $X$ is now a minimal model of $\formulaOne \vee \formulaTwo$ that does not satisfy $\formulaTwo$ we have shown $\formulaOne \vee \formulaTwo\notnmableitW \formulaTwo$ and that (Or) fails for \( \nmableitW \).\qedhere
\end{proof}

In conformance with Theorem~\ref{main}, the model \( \pmWpq \) from Example~\ref{ex:violateOR} violates the \eqref{eq:TriangleProperty}-property and \eqref{eq:StarProperty}-property.

\MYParagraph{A Note on the Expressivity on \PDLPref with System~P.}
Next, we show that preferential models for System P reasoning are quintessentially the same as their flat (see Section~\ref{sec:background_team_based_logic}) counterpart in \CPLPref.
\begin{theorem}\label{col:triangle_flattening}
    Let \( \pmW=\tuple{\pmS,\pmL,\pmP} \) be a preferential model for \PDL such that \( \nmableitW \) satisfies System~P. Then, 
    $\formulaOne \nmableitW \formulaTwo$ iff  $\formulaOne^f \nmableitWparam{\pmW'} \formulaTwo^f$,
    where \( \pmW'=\tuple{\pmS',\pmL',\pmP'} \) denotes the preferential model for \CPL induced by $\pmW$, i.e.,
    one obtains \( \pmW' \) from \( \pmW \) by first remove all states labelled by non-singleton teams, then replace labels of the singleton teams by their content.
\end{theorem}
\begin{proof}
    By construction, for all valuations $s,s'$ it holds that $s\prec' s'$ if and only if  $\{s\}\prec \{s'\}$.  By Theorem~\ref{main}, $W$ satisfies the \eqref{eq:TriangleProperty}-property and hence the minimal elements of $\prec$ are singleton teams. Hence 
    $\formulaOne\nmableit \formulaTwo$, iff,  %
    for all minimal $\{s\}\in\modelsOf{\formulaOne}$  with $\{s\}\models \formulaTwo$, iff,
    for all $\prec'$-minimal $s\in \modelsOf{\formulaOne^f}:$  $s\models \formulaTwo^f$.
    The last equivalence holds due to the remark on flattening after Proposition~\ref{prop:pdl_pincl_properties}.
\end{proof}
Theorem~\ref{col:triangle_flattening} demonstrate that System~P reasoning in \PDLPref is not fully employs the underlying team-semantics.

\section{\hspace{-0.5em}Preferential Reconstruction of \PDL and \CPL}
\label{sec:prefReconstruct}

In this section, we characterize logical entailment \( \models^{t} \) for \PDL, as well as the logical entailment  \( \models^c \) for propositional logic with classical semantics in a non-trivial canonical way.
Let \( \pmWsub = \tuple{\pmSsub,\pmLsub,\pmPsub} \) and \( \pmWsup = \tuple{\pmSsup,\pmLsup,\pmPsup} \) be the preferential models such that the following holds:
\begin{align*}
    &\pmSsub = \pmSsup = \{ s_{\Int} \mid \Int \text{ is a non-empty team} \} \\
    &\pmLsub(s_X) = \pmLsup(s_X)   = X \\ 
    & Y \pmPsub X  \text{ if } Y \subsetneq X \qquad
    Y \pmPsup X  \text{ if } X \subsetneq Y
\end{align*}
In \( \pmWsub \) and \pmWsup, for each team \( X \) there is exactly one state \( s_X \) that is labelled by \( X \).
In \( \pmPsub \), subsets of a team are preferred, whereas in \( \pmPsup \) superset teams are preferred.

The preferential model \( \pmWsup \) gives rise to the \PDL entailment relation \( \models \), and the preferential model \( \pmWsup \) gives rise to \CPL entailment of the flattening \( \models^c \).

\pagebreak[3]
\begin{proposition}\label{prop:pdl_examples}
    For all \pdl-formulas \( \formulaOne , \formulaTwo \) we have:
    \begin{enumerate}[(1)]
        \item \( \formulaOne \nmableitWparam{\pmWsub} \formulaTwo \text{ if and only if }  \formulaOne^f \models^c \formulaTwo^f \)
        \item \( \formulaOne \nmableitWparam{\pmWsup} \formulaTwo \text{ if and only if }  \formulaOne \models \formulaTwo \)
    \end{enumerate}
\end{proposition}
\begin{proof}We show statements (1) and (2).
        \emph{(1)}
        Observe at first that we have \( \formulaOne \nmableitWparam{\pmWsub} \formulaTwo \) exactly when we also have \( \minOf{\modelsOf{\formulaOne}}{\pmPsub} \subseteq \modelsOf{\formulaTwo} \).
        Because \pdl has the downwards closure property, we also have that stating \( \minOf{\modelsOf{\formulaOne}}{\pmPsub} \subseteq \modelsOf{\formulaTwo} \) is equivalent to stating that for all singleton teams \( \{v\} \) holds that \( \{v\} \models \formulaOne \)  implies \( \{v\} \models \formulaTwo \).
        The latter statement is equivalent to stating that for the flattening \( \formulaOne^f \) and \( \formulaTwo^f \) holds that for all valuations \( v \) holds that \( v \models \formulaOne^f \)  implies \( v \models \formulaTwo^f \) (see also Section~\ref{sec:background_team_based_logic}). Hence, we have \( \formulaOne \nmableitWparam{\pmWsub} \formulaTwo \text{ if and only if }  \formulaOne^f \models^c \formulaTwo^f \).
        \emph{(2)}
        We obtain \( {\models}  \subseteq  {\nmableitWparam{\pmWsup}} \) immediately by the definition of \( \nmableitWparam{\pmWsup} \).
        We consider the other direction. The statement \(  \formulaOne \models \formulaTwo \)  is equivalent to \( \modelsOf{\formulaOne} \subseteq \modelsOf{\formulaTwo} \). Because \( \modelsOf{\formulaOne} \) is downward-closed, there are (pairwise \( \subseteq \)-incomparable) teams \( X_1,\ldots,X_n \) such that \( \modelsOf{\formulaOne}=\mathcal{P}(X_1)\cup\ldots\cup\mathcal{P}(X_n)\).
        Because of the last property, we have that \(  \formulaOne \models \formulaTwo \) holds exactly when \( \{X_1,\ldots,X_n\} \subseteq \modelsOf{\formulaTwo} \) holds.
        By construction of \( \pmWsup \) we have \( {\minOf{\modelsOf{\formulaOne}}{\pmPsup}}=\{X_1,\ldots,X_n\}  \) for \( \formulaOne \).
        Consequently, we also have that \( \formulaOne \nmableitWparam{\pmWsup} \formulaTwo \) holds and consequently, we also have \( {\nmableitWparam{\pmWsup}}  \subseteq   {\models} \). \qedhere
\end{proof} 
\section{Implications for \TPLPref}
\label{sec:discussion}
We consider the preferential version of the fragment {\TPL} of \PDL.
Inspecting Example~\ref{ex:violateOR} the proof of Proposition~\ref{prop:systemCteams} reveals that they also apply to \TPLPref.
\begin{proposition}\label{prop:systemCteamsTPL}
    \TPLPref satisfies System~C and violates System~P.
\end{proposition}
Surprisingly, Theorem~\ref{main} does not carry over to \TPLPref.
We consider the in the following an example that witnesses that  \eqref{eq:StarProperty} and \eqref{eq:TriangleProperty} do not characterize System~P in \TPLPref.
\begin{example}\label{ex:violateStar}
    Assume that \( N=\{p\}\subseteq\Prop \) holds.
    There are exactly two valuations \( v_{p} \) and \( v_{\overline{p}}  \) with \( v_{p}(p) = 1 \) and \( v_{\overline{p}}(p) = 0 \).
    We consider the teams \( X_{p} = \{ v_{p} \} \), \( X_{\overline{p}}=\{ v_{\overline{p}}  \} \), and  \( X_{p\overline{p}}=\{ v_1,v_2\} \).
    Let \( \pmW_{\oast} = \tuple{\pmS_{\oast} , \pmL_{\oast} , \pmP_{\oast} } \) be the preferential model such that
    \begin{align*}
        \pmS_{\oast}  & = \{ s_{X_{p}} ,  s_{X_{\overline{p}}}, s_{X_{p\overline{p}}} \} & \pmL_{\oast}(s_X) & = X         
    \end{align*}
    holds, and such that \( \pmPpq \) is the strict partial order given by (for the sake of readability, we identify \( s_X \) with \( X \)):
    \begin{align*}
        X_{p\overline{p}} & \pmP_{\oast}  X_{p}       &        X_{p\overline{p}} & \pmP_{\oast}  X_{\overline{p}} 
    \end{align*}
    One can check that for all \PL-formulas the postulates \eqref{pstl:Or} is satisfied. System~P satisfaction follows then from Proposition~\ref{prop:systemCteamsTPL}.
    Clearly, \( X_{p\overline{p}} \pmP_{\oast}  X_{p} \) witness a violation of the \eqref{eq:TriangleProperty}-property.
    For a violation of the \eqref{eq:StarProperty}-property, we make the following observation:
    \begin{align*}
        \minOf{\modelsOf{p\lor\neg p}}{\pmP_{\oast}} & = \{ X_{p\overline{p}} \} \not\subseteq \{ X_{p} \} \cup \{ X_{\overline{p}} \} \\
        & = \minOf{\modelsOf{p}}{\pmP_{\oast}} \cup \minOf{\modelsOf{\neg p}}{\pmP_{\oast}} 
    \end{align*}
\end{example}
In summary, we obtain from Example~\ref{ex:violateStar} the following.
\begin{proposition}
   \( \pmW_{\oast} \) is a preferential model for \TPL that violates \eqref{eq:StarProperty} and \eqref{eq:TriangleProperty}, yet \( \nmableitWparam{\pmW_{\oast} } \) satisfies System~P.
\end{proposition}
    Unfortunately, a corresponding alternative for \eqref{eq:StarProperty} and \eqref{eq:TriangleProperty} that characterizes System~P within \TPL eludes us so far.
    
    One might remark that \( \pmW_{\oast} \) is also a preferential model for \PDL. However, in compliance with Theorem~\ref{main}, when have that System~P is violated by \( \nmableitWparam{\pmW_{\oast}} \) in \PDL when setting \( \formulaOne \), \( \formulaTwo\) and \(\formulaThree  \) in \eqref{pstl:Or} to \(  \dep{p} \). We have \( \dep{p} \nmableitWparam{\pmW_{\oast}} \dep{p} \) and \( \dep{p} \lor \dep{p} \notnmableitWparam{\!\!\!\pmW_{\oast}} \dep{p} \); the latter because \( \minOf{\modelsOf{\dep{p}\lor\dep{p}}}{\pmP_{\oast}} = \{ X_{p\overline{p}}\ \} \) and  \( X_{p\overline{p}} \not\models \dep{p} \) holds. Note that \( \dep{p} \) is not a \PL-formula, and thus, also not of \TPL.

\section{Complexity of Entailment}
\label{sec:complexity}
In this section, we study the computational complexity of entailment-related problems. 
Informally this means, that we are given a preferential model $\pmW$ and to formulas $\formulaOne, \formulaTwo$ as inputs. 
Then we ask whether $\formulaOne\nmableitW\formulaTwo$ is true.

\MYParagraph{Preferential Propositional Logic.}
Preferential models $\pmW$ encompass three components: a set $\pmS$, a labelling function $\pmL$, and an order $\pmP$. 
Let us first define the problem of interest.
\problemdef{$\pmCircOrderModelChecking$ --- entailment problem for preferential propositional logic}{A finite preferential model $\pmW = \tuple{\pmS,\pmL,\pmP}$ for \CPL and $\formulaOne,\formulaTwo\in\PL$}{Is it true that $\formulaOne\nmableitW\formulaTwo$}
First, we see how this problem can be solved in polynomial time in a brute-force approach.
\begin{lemma}
    $\pmCircOrderModelChecking\in\Ptime$. \label{lem:pentpl-p}
\end{lemma}
\begin{proof}
    Consider an input $\pmW,\formulaOne,\formulaTwo$ with $\pmW=\tuple{\pmS,\pmL,\pmP}$ as defined above. 
    We construct a polynomial-time algorithm in the following:
    \begin{enumerate}
        \item Check for every $s\in\pmS$ whether $s\models\formulaOne$ and place a mark in $\pmS$ at this element if yes. 
        \item In the corresponding graph $(\pmS,\pmP)$ search for all marked minimal elements $s$ and check if $s\models\formulaTwo$. If not, reject.
        \item Accept.
    \end{enumerate}
    The $\models$ checks are in $\NC1\subseteq\Ptime$~\cite{DBLP:journals/siamcomp/BussCGR92}. 
    The minimum search is a simple graph search for minimal elements in DAGs, which can be done in time linear in the size of the graph $(\pmS,\pmP)$.
\end{proof}
\begin{lemma}
    $\pmCircOrderModelChecking\in\Ptime$ is $\NC{1}$-hard under $\leq_m^{\AC{0}}$-reductions.\label{lem:pentpl-nc1-hard}
\end{lemma}
\begin{proof}
    The model checking problem for $\CPL$ is $\NC{1}$-complete \cite{DBLP:journals/siamcomp/BussCGR92}. 
    Given a propositional assignment $\theta$ and a propositional formula $\formulaOne$, reduce it quite directly as follows showing $\NC{1}$-hardness:
    \[
    (\theta,\formulaOne)\mapsto ((\{\theta\},\mathrm{id}_\pmS,\emptyset),\top,\formulaOne).
    \]
    As $\pmS$ contains only one element which also is satisfied by the first formula $\top$, we require it to satisfy also $\formulaOne$. 
    This is equivalent to the model checking problem for $\CPL$. 
    The reduction is a sheer mapping of values or constant parts, so computable by an $\AC{0}$ circuit family.
\end{proof}
\begin{theorem}
    $\pmCircOrderModelChecking\in\Ptime$ and $\NC{1}$-hard under $\leq_m^{\AC{0}}$-reductions.\label{thm:pent-pl}
\end{theorem}
\begin{proof}
    Proven by Lemmas~\ref{lem:pentpl-p} and \ref{lem:pentpl-nc1-hard}.
\end{proof}

Clearly, there are exponentially many assignments in the number of variables of a considered formulas. 
This can easily result in an exponentially large set $\pmS$ and hides the ``real'' complexity of the problem. 
Accordingly, we want to consider a more succinct version of the problem. 
We approach this observation with the following definition.

\begin{definition}
    Let $N\subseteq\Prop$ be a set of propositions with $|N|=n$, $\pmS=\{0,1\}^m$ be a set for $m\in n^{O(1)}$, 
    and ${\pmP}\subseteq\pmS\times\pmS$ be a strict partial order. 
    Now let $\pmW$ be a preferential model $\pmW=\tuple{\pmS,\pmL,\pmP}$ such that $\pmL\colon\pmS\rightharpoonup\allModels{N}$ is a partial labelling. 
    Let there be two $n^{O(1)}$-sized circuit families $\mathcal L,\mathcal O$ (labelling, ordering) such that the following is true:
    \begin{enumerate}
        \item $\ell$ is computed by $\mathcal L$,
        \item $\mathcal O\colon\pmS\times\pmS\rightharpoonup\{0,1\}$ is a partial function such that for $s,s'\in \pmS$, the circuit outputs $1$ if and only if $s\pmP s'$ is true.
    \end{enumerate}
    We call $(\mathcal L,\mathcal O)$ an \emph{$n^{O(1)}$-sized circuit representation} of $\pmW$.
\end{definition}
\begin{remark}
    Notice that the size of $\pmS$ is always $2^m$. 
    However, the two circuit families $\mathcal L$ and $\mathcal O$ need to deal with so-to-speak irrelevant input strings in a reasonable way. 
    In this light, the preimage of the partial function $\ell$ induces what part of $\pmS$ is relevant. 
\end{remark}

Now, let $\succinct\pmCircOrderModelChecking$ be the problem considering only instances that have a $n^{O(1)}$-sized circuit representation of the preferential model, i.e., the input then is of the form $\tuple{\tuple{\mathcal O,\mathcal L},\formulaOne,\formulaTwo}$.

\begin{lemma}
    $\succinct\pmCircOrderModelChecking$ is $\Delta^p_2$-hard under $\leqlogm$-reductions.\label{lem:COMC-deltap2-hard}
\end{lemma}
\begin{proof}
    We state a reduction from $\OLMSco$ to $\succinct\pmCircOrderModelChecking$. 
    By virtue of Ex.~\ref{ex:lex-circ}, there exists a polynomial-sized (in the number of variables) circuit family (even in $\AC0$) that defines the lexicographic order on binary strings $\lex$. 
    Now, just swap the inputs of this circuit and thereby define the lexicographic order $\rlex$. 
    Call this circuit $\mathcal O$. 
    Let $\formulaOne(x_1,\dots,x_n)$ be the input of $\OLMSco$. 
    Regarding the circuit representation of the preferential model, we let $m=n$, so $\pmS=\allModels{N}$ where $N=\{x_1,\dots,x_n\}$. 
    As a result, $\mathcal L=\id_\pmS$, where $\id_\pmS$ is the identity function on $\pmS$.
    Then, we define the reduction 
    \[
    \formulaOne\overset{f}{\mapsto} ((\mathcal L,\mathcal O),\formulaOne,\neg x_n)
    \]
    
    Example~\ref{ex:lex-circ} shows that $\mathcal O$ is logspace-constructible and thereby $f$ is logspace-constructible. 
    
    We claim that the reduction is correct, i.e., $\formulaOne\in\OLMSco$ if and only if $f(\formulaOne)\in\succinct\pmCircOrderModelChecking$. 
    
    ``$\Rightarrow$'':
    Let $\formulaOne$ be a positive instance of $\OLMSco$ and $\theta$ be the $\rlex$-maximal satisfying assignment. 
    Then there are two possibilities. 
    \begin{enumerate}
        \item 
        $\formulaOne$ is unsatisfiable. Then, because \( \nmableitW \) suffers logical explosion all formulas are implied and \( \formulaOne \nmableitW \neg x_n \) is true. 
        \item 
        $\formulaOne$ is satisfiable, but for the $\rlex$-minimal assignment $\theta$ (notice that this is the $\lex$-largest assignment) we have that $\theta(x_n)=0$.
        Hence, $\theta$ is not a model for $x_n$ and thereby $\formulaOne\nmableitW \neg x_n$ holds.
    \end{enumerate}
    
    ``$\Leftarrow$': 
    Let $\formulaOne\notin\OLMSco$.   
    Clearly, $\theta(x_n)=1$ by requirement of $\OLMSco$.  
    Then, the string representation of $\theta$ is minimal w.r.t.\ $\rlex$. 
    Furthermore, for any model that satisfies $x_n$ we have that $x_n$ is assigned $1$, hence, also $\theta$ is in the set of models of $x_n$. 
    As a result $\formulaOne\notnmableitW \neg x_n$, and $f(\formulaOne)\notin\succinct\pmCircOrderModelChecking$.
\end{proof}

For some $\pmP$ order, we write $\succinct\pmCircOrderModelChecking_{\pmP}$ for the problem $\succinct\pmCircOrderModelChecking$ where the order for given instances is fixed to $\pmP$.

\begin{lemma}
    $\succinct\pmCircOrderModelChecking_{\rlex}$ is in $\Delta^p_2$.
    \label{lem:COMC-in-deltap2}
\end{lemma}
\begin{proof}
    Regarding membership in $\Delta_2^p$, we sketch a polynomial time algorithm that uses an oracle for proposition satisfiability. 
    Let $\formulaOne,\formulaTwo$ be the input formulas and $(\mathcal L,\mathcal O)$ be the circuit-representation of the preferential model. 
    Then, we use the SAT oracle as follows. 
    
    For $c\in\{0,1\}$, we let $\formulaOne(x_i=c)$ be the formula $\formulaOne$ where every occurrence of $x_i$ is substituted by the value of $c$.
    Now the $\Delta_2^p$-algorithm works as follows. 
    Ask the SAT oracle if $\formulaOne(x_1=0)$ is satisfiable. 
    If yes, then proceed similarly with $x_2$ for $\formulaOne(x_1=0)$. 
    If no, then proceed similarly with $x_2$ for $\formulaOne(x_1=1)$. 
    In the end, we know the lexicographic maximal assignment and need to merely check if it satisfies $\formulaTwo$.
\end{proof}
Because of the previous result, we now have of a complete problem regarding the specific order $\rlex$.

\begin{theorem}
    $\succinct\pmCircOrderModelChecking_{\rlex}$ is $\Delta^p_2$-complete under $\leqlogm$-reductions.\label{thm:COMC-deltap2-complete}
\end{theorem}
\begin{proof}
    Lemmas~\ref{lem:COMC-in-deltap2} and \ref{lem:COMC-deltap2-hard} prove the theorem.
\end{proof}
\begin{lemma}
    $\succinct\pmCircOrderModelChecking$ is in $\Pi^p_2$.\label{lem:succ-PL-membership}
\end{lemma}
\begin{proof}
    Let $\tuple{\tuple{\mathcal O,\mathcal L},\formulaOne,\formulaTwo}$ be the input. 
    W.l.o.g., assume that the set of propositions $N=\{x_1,\dots,x_n\}$. 
    We describe the behaviour of the $\Pi^p_2$-machine that decides the problem.
    
    \begin{enumerate}
        \item Univerisally nondeterministically branch on all elements $j\in\pmS$ specified by inputs to $\mathcal O$, and all assignments $s\colon\{x_1,\dots,x_n\}\to\{0,1\}$. 
        \item Existentially nondeterministically branch on all assignments $s'\colon\{x_1,\dots,x_n\}\to\{0,1\}$.
        \item If $j\neq\mathcal L(s)$ then accept.
        \item If $j\models\formulaTwo$ then accept.
        \item If $\mathcal L(s')\models\formulaOne$ and $\mathcal O(s',s)$ then accept.
        \item Reject.
    \end{enumerate}
    The nondeterminism induced by 1./2.\ is $\forall\exists$, hence $\Pi^p_2$. 
    Steps~3.\ and 5.\ make use of the circuit family $\mathcal L$ resulting in a $\Ptime$ computation. 
    Again, the $\models$-checks in 4./5.\ are in $\NC{1}$~\cite{DBLP:journals/siamcomp/BussCGR92}. 
    The computation of the circuit value $\mathcal O(s,s')$ in 5.\ is in $\Ptime$.
    Reaching 6.\ means that 
    \[
    (j=\ell(s))\land(j\not\models \formulaTwo)\land (\ell(s')\not\models\formulaOne\lor s'\not\pmP s)
    \]
    resulting in a negative answer to the input.
\end{proof}

\begin{theorem}
    $\succinct\pmCircOrderModelChecking$ is in $\Pi^p_2$ and $\Delta^p_2$-hard w.r.t.\ $\leqlogm$-reductions.\label{thm:suc-pent-pl-upper-lower}
\end{theorem}

\MYParagraph{Preferential Propositional Dependence Logic.}
The following version of $\pmCircOrderModelChecking$ lifts the problem to team semantics and the logic $\PDL$ by similar definitions. 

\problemdef{$\pmCircOrderModelCheckingTS$ --- entailment problem for preferential propositional dependence logic}{A finite preferential model $\pmW=\tuple{\pmS,\pmL,\pmP}$ for \PDL and $\formulaOne,\formulaTwo\in\pdl$}{Is it true that $\formulaOne\nmableitW\formulaTwo$}
In the following, we will state a result regarding the less known complexity class $\Theta_2^p$.
This class is defined as $\Ptime^{\NP[\log]}$ meaning a restriction to logarithmic many calls to the $\NP$ 
oracle. 
By definition, we then have the containment $\Theta_2^p\subseteq\Delta_2^p$. 
Also it can be characterised by $\Ptime^{||\NP}$ which is having non-adaptive but unrestricted many parallel $\NP$ oracle calls~\cite{DBLP:journals/iandc/BussH91,DBLP:journals/jcss/Hemachandra89}.

\begin{lemma}
    $\pmCircOrderModelCheckingTS$ is in $\Theta_2^p$.\label{lem:pentpdl-member}
\end{lemma}
\begin{proof}
    We present a $\Theta_2^p$ algorithm deciding the problem. 
    The model checking problem for $\PDL$ is $\NP$-complete \cite[Thm.~1]{DBLP:conf/sofsem/EbbingL12}.
    We use it as an oracle here.
    \begin{enumerate}
        \item In parallel, ask the $\NP$-oracle for each team $T\in\pmS$ whether $T\models\formulaOne$ and $T\models\formulaTwo$.
        \item For every minimal element in the order induced graph $(\pmS,\pmP)$, if the oracle answers where of the form $(1,0)$ (that is, $\formulaOne$ was satisfied but $\formulaTwo$ not) then reject.
        \item Accept.
    \end{enumerate}
    An oracle answer does not imply a different call afterwards.
    As a result, the oracle calls are non-adaptive and can be asked in parallel. 
    As the input consists of $\pmS$, we have enough time to browse through all elements which also allows of identifying the minimal elements in the graph $(\pmS,\pmP)$. 
    The algorithm is correct as step 3.\ is executed if no contradiction of the form that a minimal assignment satisfies $\formulaOne$ but falsifies $\formulaTwo$ occurs.
    This completes the proof.
\end{proof}

\begin{lemma}
    $\pmCircOrderModelCheckingTS$ is $\NP$-hard w.r.t.\ $\leqlogm$-reduct\-ions.\label{lem:PDL-lowerbound}
\end{lemma}
\begin{proof}
    The model checking problem for $\PDL$ is $\NP$-complete \cite[Thm.~1]{DBLP:conf/sofsem/EbbingL12}. 
    Now reduce it quite directly as follows showing $\NP$-hardness: 
    \[
    (T,\formulaOne)\mapsto ((\{T\},\mathrm{id}_\pmS,\emptyset),\top,\formulaOne).
    \]
    As $\pmS$ contains only one element which also is satisfied by the first formula $\top$, we require it to satisfy also $\formulaOne$. 
    This is equivalent to the model checking problem for $\PDL$.
\end{proof}
\begin{theorem}
    $\pmCircOrderModelCheckingTS$ is in $\Theta_2^p$ and $\NP$-hard w.r.t.\ $\leqlogm$-reductions.\label{thm:pent-pdl}
\end{theorem}
\begin{proof}
    Shown by Lemmas~\ref{lem:PDL-lowerbound} and \ref{lem:pentpdl-member}.
\end{proof}

Analogously as before, we assume for the succinct version $\succinct\pmCircOrderModelCheckingTS$, that the circuit families now are of size $(2^{n})^{O(1)}$ (saving one exponential step via succinct representations), where $n$ is the number of variables in $\formulaOne$ and $\formulaTwo$. 
Notice that, while the inputs can be still of exponential size in $n$ (a single team can have this size), it is still meaningful to have smaller inputs (avoiding doubly exponentially many such teams as trivial bound for $|\pmS|$). 

\pagebreak[3]
\begin{lemma}
    $\succinct\pmCircOrderModelCheckingTS$ is in $\Pi^p_2$. \label{lem:sucpent-pdl-in-pip2}
\end{lemma}
\begin{proof}
    The algorithm stated in the proof of Lemma~\ref{lem:succ-PL-membership} works if it is modified as follows. 
    While Steps~4.\ and 5.\ have a higher model checking complexity for $\PDL$, that is, $\NP$ compared to $\NC{1}$, we can guess the required certificates (to obtain membership in $\Ptime$) also in Step.\ 2. 
    This does not increase complexity, as Step.\ 2 is already existential. 
    That is why this still yields a $\Pi^p_2$ algorithm.
\end{proof}
It might be a bit surprising at first sight that having a harder model checking problem does not increase the complexity, but as stated in the proof above, this is compensated by the $\forall\exists$ structure of $\Pi^p_2$.

\begin{lemma}
    $\succinct\pmCircOrderModelCheckingTS$ is $\Delta^p_2$-hard w.r.t.\ $\leqlogm$-reductions. \label{lem:sucpent-pdl-deltap2-hard}
\end{lemma}
\begin{proof}
    Use a circuit family that encodes $\pmS$ as a set of singleton teams: $\pmS\coloneqq\{\;\{v\}\mid v\colon\{p_1,\dots,p_n\}\to\{0,1\}\;\}$. 
    Flatness and the proof of Lemma~\ref{lem:COMC-deltap2-hard} then yields the result.
\end{proof}

\begin{theorem}
    $\succinct\pmCircOrderModelCheckingTS$ is in $\Pi^p_2$ and $\Delta^p_2$-hard w.r.t.\ $\leqlogm$-reductions.\label{thm:suc-pent-pdl-upper-lower}
\end{theorem}
\begin{proof}
    Shown by Lemmas~\ref{lem:sucpent-pdl-deltap2-hard} and \ref{lem:sucpent-pdl-in-pip2}.
\end{proof}

\MYParagraph{Preferential Propositional Logic with Team Semantics.}
From Theorem~\ref{thm:pent-pdl} and \ref{thm:suc-pent-pdl-upper-lower},  we can deduce similar ones for the team logic without dependence atoms $\TPL$. 
Note, that the model checking problem for $\TPL$ is (potentially) easier than for $\PDL$, as has been shown in $\Ptime$~\cite[Tab.~1]{DBLP:conf/sofsem/EbbingL12}. 
As a result, the influence of the $\NP$ completeness of model checking for $\PDL$, needs to be reconsidered.
\begin{corollary} The following holds:
    \begin{enumerate}
        \item $\pmCircOrderModelCheckingTPL$ is in $\Ptime$ and $\NC{1}$-hard under $\leq^{\AC{0}}_m$-reductions.
        \item $\succinct\pmCircOrderModelCheckingTPL$ is in $\Pi_2^p$ and $\Delta_2^p$-hard under $\leqlogm$-reductions.
    \end{enumerate}\label{cor:pent-tpl}
\end{corollary}

\section{Conclusion}
\label{sec:conclusion}
We have established a foundation for preferential
non-monotonic reasoning within the framework of team semantics. Our results also provide new insights into the algorithmic properties of classical propositional logic in the KLM framework.

Team-based logics have a wide range of applications, e.g., in the formal semantics of natural language—for instance, in the semantics of statements and questions (as in inquisitive logic) and in modeling free choice inferences (as addressed by BSML modal logics \cite{Aloni-2022}). We expect that our results will yield further applications in this domain.

In future work, we intend to investigate the query and data complexity of the problems discussed in Section~\ref{sec:complexity}, and find tight complexity bounds for the problems studied. 
This includes to investigate the complexity of a team-variant of \OLMS.
Another area of future work is the axiomatics of preferential reasoning in the team-based semantics, e.g., inclusion logic.

\clearpage
\bibliographystyle{kr}
\bibliography{merged.bib}

\end{document}